\documentclass[conference]{IEEEtran}
\IEEEoverridecommandlockouts
\usepackage[T1]{fontenc}
\usepackage{newtxtext}
\usepackage{newtxmath}
\usepackage{cite}
\usepackage{amsmath}
\usepackage{graphicx}
\usepackage{booktabs}
\usepackage{multirow}
\usepackage{xcolor}
\usepackage{tikz}
\usepackage{pgfplots}
\pgfplotsset{compat=1.18}
\usetikzlibrary{arrows.meta,positioning,calc}
\usepackage{url}
\usepackage{textcomp}
\usepackage{algorithm}
\usepackage{algpseudocode}
\usepackage[colorlinks=true,linkcolor=black,citecolor=black,urlcolor=linkblue,bookmarks=false]{hyperref}

\input{numbers.tex}

\definecolor{slotblue}{HTML}{2A78D6}
\definecolor{slotorange}{HTML}{EB6834}
\definecolor{slotaqua}{HTML}{1BAF7A}
\definecolor{slotyellow}{HTML}{EDA100}
\definecolor{slotmagenta}{HTML}{E87BA4}
\definecolor{slotviolet}{HTML}{4A3AA7}
\definecolor{linkblue}{HTML}{1C5CAB}
\definecolor{inkmuted}{HTML}{52514E}
\definecolor{gridgray}{HTML}{E4E3DF}
\definecolor{edgegray}{HTML}{B9B7B0}

\newcommand{\rift}{RIFT}
\newtheorem{proposition}{Proposition}
\newtheorem{lemma}{Lemma}
\begin{document}
\bstctlcite{IEEEexample:BSTcontrol}

\title{RIFT: Token Influence Follows\\the Path of Least Resistance}

\author{\IEEEauthorblockN{Tingan Jin}
\IEEEauthorblockA{University of California, Los Angeles\\
Los Angeles, CA, USA}
\and
\IEEEauthorblockN{Shuhang Dong}
\IEEEauthorblockA{University of California, Los Angeles\\
Los Angeles, CA, USA}
\and
\IEEEauthorblockN{Haosong Li}
\IEEEauthorblockA{Independent Researcher}}

\maketitle

\begin{abstract}
Which context tokens drive a language model's answer? Graph-based tools such as attention rollout and attention flow answer this from a single forward pass, yet their faithfulness in decoder-only models is rarely tested against interventions. We view the layers of a decoder as a causal network whose conductances are attention weights, and read a token's influence on the answer from the resistance between them. We propose \rift{} (Resistance-based Information Flow in Transformers), which scores a token by the series resistance of its \emph{path of least resistance} to the answer, found exactly by a classic shortest-path dynamic program. On synthetic in-context retrieval tasks and extractive question answering over Wikipedia paragraphs, \rift{} predicts the effect of replacing a token with a Spearman correlation of \val{C1-plr-min}--\val{C1-plr-max}; predicting a single replacement, it reaches \val{C1-lfl-min}--\val{C1-lfl-max}\% of the agreement between two different random replacements. It needs no gradients and no extra forward passes, scoring every token of a prompt in \val{cost-rift-min-ms}--\val{cost-rift-max-ms}\,ms on one CPU thread, whereas the replacement experiment needs a forward pass per token. On \val{C1-total-prompts} held-out prompts from four models of two families (Qwen3 and Llama 3.2), it outperforms token distance, attention rollout and attention flow in every pre-specified comparison, by \val{C1-gapu-min} to \val{C1-gapu-max} in mean Spearman correlation, and it is the best or tied-best of all attention-based scores we test. Aggregates that add up all routes (attention flow, effective resistance and rollout) instead mostly encode token position: as a simple proposition shows, early tokens have many more parallel routes to the answer.
\end{abstract}


\section{Introduction}

When a language model answers a question about its context, which context tokens did the answer depend on? Token-level attribution supports debugging, auditing and trust in systems built on large language models (LLMs); in retrieval-augmented and knowledge-grounded question answering, it tells which retrieved facts an answer actually used. Gradient-based and perturbation-based attributions answer the question for a given output, but require backward passes or many forward passes~\cite{ancona2018gradients,meng2022rome,zhang2024patching,cohenwang2024contextcite}. A popular cheaper family reads influence directly off the attention graph of one forward pass: \emph{attention rollout} and \emph{attention flow}~\cite{abnar2020flow}, together with contribution-aware refinements~\cite{kobayashi2020norms,ferrando2022alti,ferrando2024routes}. These tools are routinely used to visualize what a model ``looks at'', but in decoder-only LLMs their agreement with causal effects has rarely been measured, and the few existing checks report moderate to weak agreement for rollout~\cite{pramanik2026heta,lan2025macs}.

We revisit the question through an electrical analogy. Unrolling a decoder over its layers gives a causal network: each node is a token after a layer, and each edge carries the weight with which one token reads another, plus the residual connection. A source token's influence on the answer can then be read from the \emph{resistance} between the source and the answer position. Two classical notions of resistance lead to two families of scores (Fig.~\ref{fig:schematic}). The \emph{effective} resistance lets current spread over all routes; like attention flow and rollout, it aggregates many paths. The \emph{path of least resistance} instead asks how easy the single best route is. Our method, \rift{} (Resistance-based Information Flow in Transformers), scores each token by this path. Finding it is a classic shortest-path search, solved exactly by dynamic programming, so \rift{} brings a tool of classic AI to the interpretation of generative models.

The two families behave very differently, which is the main message of this paper:
\begin{itemize}
\item \textbf{\rift{} predicts causal influence.} \rift{} reaches a Spearman correlation of \val{C1-plr-min}--\val{C1-plr-max} with the effects of token-replacement interventions. It predicts the effect of a single replacement nearly as well as a second, different random replacement of the same token does (\val{C1-lfl-min}--\val{C1-lfl-max}\% of that agreement). On \val{C1-total-prompts} held-out prompts from four models of two families, including extractive question answering on Wikipedia text, it beats token distance, attention rollout and attention flow in every pre-specified comparison, by \val{C1-gapu-min} to \val{C1-gapu-max} in mean Spearman correlation.
\item \textbf{\rift{} is the best attention-based score we test, or tied for best.} It is best or statistically tied for best in every dataset. Its single-route sibling, the widest path, ties it in \val{C1-wid-tie-of} datasets (Fig.~\ref{fig:main}); the strongest single attention weight and a multiplicative path score trail it in most datasets.
\item \textbf{All-path aggregates track position, not content.} In the Qwen3 models, attention flow ranks tokens exactly from earliest to latest in \val{Q-flow-exact} of \val{Q-flow-n} prompts, and in Llama its Spearman correlation with token proximity (inverse distance to the answer) is still \val{L-pos-flow-min} to \val{L-pos-flow-max}; negative values mean that earlier tokens score higher. Across the confirmation datasets, effective resistance (\val{C1-pos-kirch-min} to \val{C1-pos-kirch-max}) and rollout (\val{C1-pos-roll-min} to \val{C1-pos-roll-max}) follow position too.
\item \textbf{\rift{} is cheap and interpretable.} It needs no gradients and no forward pass beyond the one that produced the answer: one sweep of a dynamic program over the attention weights, from the last layer down, scores every token of a prompt in \val{cost-rift-min-ms}--\val{cost-rift-max-ms}\,ms on one CPU thread, whereas measuring replacement effects takes a forward pass per token (Fig.~\ref{fig:cost}). The best route from the token that holds the answer reaches the answer position in a single jump in \val{route-single-min}--\val{route-single-max}\% of prompts, for Qwen3 on the synthetic tasks mostly at one layer (\val{route-q-mode-pct}\% of prompts), pointing to a candidate read-out layer.
\end{itemize}
Our contributions are (i) a resistance view of token influence that casts rollout, attention flow and effective resistance as all-path aggregates, and \rift{}, a score based on the path of least resistance; (ii) a causal evaluation protocol with a label-reliability reference and a pre-specified, held-out confirmation; (iii) evidence, on synthetic and natural text, that \rift{} predicts causal influence far better than the standard baselines, at a cost of milliseconds per prompt; and (iv) an analysis of why all-path aggregates favor early tokens in decoder LLMs.

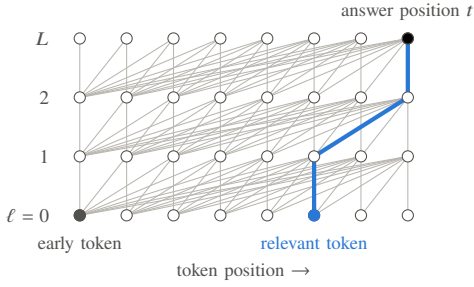
\begin{figure}[t]
\centering
\begin{tikzpicture}[x=0.62cm,y=0.78cm,
  node/.style={circle,draw=inkmuted,fill=white,inner sep=0pt,minimum size=4.2pt,line width=0.4pt},
  lbl/.style={font=\scriptsize,text=inkmuted}]
  \foreach \l in {0,1,2}{
    \foreach \a in {1,...,8}{
      \foreach \b in {1,...,8}{
        \ifnum\b<\a\relax\else\draw[edgegray,line width=0.4pt] (\a,\l) -- (\b,\l+1);\fi
      }
    }
  }
  \draw[slotblue,line width=1.6pt] (6,0) -- (6,1) -- (8,2) -- (8,3);
  \foreach \x in {1,...,8}{\foreach \y in {0,1,2,3}{\node[node] at (\x,\y) {};}}
  \node[node,fill=slotblue,draw=slotblue] at (6,0) {};
  \node[node,fill=black] at (8,3) {};
  \node[node,fill=inkmuted,draw=inkmuted] at (1,0) {};
  \foreach \y/\t in {0/{$\ell=0$},1/{$1$},2/{$2$},3/{$L$}}{\node[lbl,anchor=east] at (0.55,\y) {\t};}
  \node[lbl,anchor=north] at (1,-0.2) {early token};
  \node[lbl,anchor=north,text=slotblue] at (6,-0.2) {relevant token};
  \node[lbl,anchor=south] at (8,3.15) {answer position $t$};
  \node[lbl,anchor=north,align=center] at (4.5,-0.65) {token position $\rightarrow$};
\end{tikzpicture}
\caption{The resistance view. A decoder unrolled over layers is a causal network; node $(\ell,i)$ is token $i$ after layer $\ell$, and edges carry attention mixing weights (plus the residual) as conductances. An early token reaches the answer through many parallel routes (gray), which all-path aggregates add up. \rift{} scores only the strongest single route, the path of least resistance (blue).}
\label{fig:schematic}
\end{figure}

\section{Related Work}

\textbf{Attention-based attribution.} Abnar and Zuidema~\cite{abnar2020flow} introduced attention rollout, the product of residual-augmented attention matrices, and attention flow, a maximum flow through the layered attention graph; generalized attention flow adds gradients and a barrier method~\cite{azarkhalili2025gaf}. Later work weighted attention by the norm of transformed value vectors~\cite{kobayashi2020norms}, decomposed full layer contributions~\cite{ferrando2022alti}, propagated relevance~\cite{chefer2021beyond,achtibat2024attnlrp}, and extracted, for each prediction, a subgraph of its most important edges~\cite{ferrando2024routes}; \rift{} instead scores every token by its single best route. A single maximum-probability path, with edge cost $-\log a$, has explained Transformers on tabular data~\cite{trevino2024mla}; on our data, series resistance predicts influence better than this multiplicative cost in most datasets (Section~\ref{sec:closest}). Whether attention explains predictions has been debated~\cite{jain2019attention,wiegreffe2019attention}; we sidestep the debate by scoring every method against interventions on the input, in the spirit of erasure-based faithfulness tests~\cite{deyoung2020eraser}. Gradient attributions~\cite{ancona2018gradients}, activation patching~\cite{meng2022rome,conmy2023acdc,zhang2024patching} and its gradient-based approximation~\cite{kramar2024atp}, and ablation-trained context attribution~\cite{cohenwang2024contextcite,cohenwang2025at2} measure effects more directly, at the cost of backward passes, many ablations, or training a surrogate.

\textbf{Resistance and over-squashing.} In graph neural networks, curvature~\cite{topping2022curvature}, effective resistance~\cite{black2023resistance} and topology~\cite{digiovanni2023oversquashing} bound how well information travels between nodes. Barbero et al.~\cite{barbero2024glasses} carried this view to decoder-only Transformers, where the causal mask gives early tokens more paths to the final token. We use resistance as a \emph{measurement} of influence and find that this abundance of paths is exactly what misleads all-path scores.

\textbf{Position bias and dedicated routes.} LLMs place disproportionate attention on early tokens (attention sinks)~\cite{xiao2024sinks} and use long contexts unevenly~\cite{liu2024lost}. Theory shows that causal masking biases attention toward early tokens as depth grows~\cite{wu2025positionbias}; with residual connections, causal masking alone yields broad, often U-shaped influence profiles~\cite{chowdhury2026birth,herasimchyk2026structural}, which Herasimchyk et al.\ derive from a residual-aware cumulative rollout and find to match measured input-token influence in pretrained models~\cite{herasimchyk2026structural}. Abnar and Zuidema already expected rollout and flow to favor initial tokens in decoders and suggested normalizing by the receptive field~\cite{abnar2020flow}; Metzger et al.\ normalized flow in decoders to make it position independent~\cite{metzger2022general}. We complement these analyses in three ways: we score against token-replacement interventions, which in our data are nearly independent of position; we cover attention flow and effective resistance; and we study models with rotary position embeddings. Mechanistic studies find that copying and retrieval run through a few dedicated heads~\cite{olsson2022induction,wu2024retrievalhead}, which anticipates our finding that a single strong route reflects content.

\section{A Resistance View of Token Influence}

\subsection{The layered causal network}
Consider a decoder with $L$ layers processing $T$ tokens. Node $(\ell,i)$ is token $i$ after layer $\ell$, with $\ell=0$ for the input embeddings. For every $j\le i$ there is an edge $(\ell-1,j)\to(\ell,i)$ with conductance
\begin{equation}
c_\ell(i,j)=\tfrac12\big(\bar A^{\ell}_{ij}+\delta_{ij}\big),
\label{eq:conductance}
\end{equation}
where $\bar A^{\ell}$ is the attention matrix of layer $\ell$ averaged over heads and $\delta_{ij}$ is the Kronecker delta. This is the residual-augmented attention graph used by rollout and attention flow~\cite{abnar2020flow}; each row sums to one. The resistance of an edge is $1/c_\ell(i,j)$. Section~\ref{sec:robust} also tests two contribution-aware conductances: value-vector norms~\cite{kobayashi2020norms} and the norm of each block's local Jacobian, i.e., the sensitivity of a token's output state to each input state.

\subsection{\rift{}: the path of least resistance}
Let $s$ be a source position and $t$ the answer position (the last prompt token). A path from $(0,s)$ to $(L,t)$ visits positions $s=j_0\le j_1\le\dots\le j_L=t$, one edge per layer. Its series resistance is $R(P)=\sum_{\ell=1}^{L} 1/c_\ell(j_\ell,j_{\ell-1})$, and we define
\begin{equation}
\mathrm{RIFT}(s)=\Big(\min_{P:\,(0,s)\to(L,t)} R(P)\Big)^{-1}.
\end{equation}
The minimum is a shortest-path problem on a layered graph. Algorithm~\ref{alg:plr} solves it for all sources at once, from the last layer down, in a few milliseconds for our prompts (Section~\ref{sec:cost}); keeping the argmin of each step also returns the route, i.e., the layers at which it moves between positions. Every path has exactly $L$ edges, and a single weak edge makes its resistance large, so \rift{} measures the strength of the best route including its bottleneck. The closely related \emph{widest path}, $\max_P\min_{e\in P}c_e$, scores the bottleneck alone.

\begin{algorithm}[t]
\caption{\rift{}: path of least resistance for all sources at once}
\label{alg:plr}
\begin{algorithmic}[1]
\Require conductances $c_\ell(i,j)$ from Eq.~\eqref{eq:conductance}, answer position $t$
\State $B_L(t)\gets0$; \ $B_L(i)\gets\infty$ for $i\ne t$
\For{$\ell=L,\dots,1$}
  \For{$j=1,\dots,T$}
    \State $B_{\ell-1}(j)\gets\min_{i\ge j}\big[\,1/c_\ell(i,j)+B_\ell(i)\,\big]$\label{alg:update}
  \EndFor
\EndFor
\State \Return $\mathrm{RIFT}(s)=1/B_0(s)$ for every source $s$
\end{algorithmic}
\end{algorithm}

\begin{lemma}[Exactness]
$B_\ell(j)$ is the smallest series resistance of any path from $(\ell,j)$ to $(L,t)$ ($\infty$ if there is none), so Algorithm~\ref{alg:plr} returns $\mathrm{RIFT}(s)$ exactly for every source $s$, in $O(LT^2)$ time.
\end{lemma}
\begin{IEEEproof}
Backward induction on $\ell$; at $\ell=L$ the only path is the empty one at $j=t$. A path from $(\ell-1,j)$ is an edge to some $(\ell,i)$ with $i\ge j$ followed by a path from $(\ell,i)$ to $(L,t)$, so its resistance is $1/c_\ell(i,j)$ plus at least $B_\ell(i)$, with equality for the best remainder. Line~\ref{alg:update} minimizes over all $i\ge j$, hence over all paths, with $O(T^2)$ operations per layer. Softmax weights are positive, so every source $s\le t$ gets a finite score. Replacing $(\min,+)$ by $(\max,\min)$ gives the widest path by the same argument.
\end{IEEEproof}

\textbf{Single-jump routes.} A route that stays on the source's residual stream up to layer $\ell$, jumps to $t$ there and stays on $t$ has the resistance
\begin{equation}
R_\ell(s)=\sum_{k<\ell}\frac{2}{1+\bar A^{k}_{ss}}+\frac{2}{\bar A^{\ell}_{ts}}+\sum_{k>\ell}\frac{2}{1+\bar A^{k}_{tt}}.
\label{eq:jump}
\end{equation}
Each residual step costs between 1 and 2, whereas the jump costs at least 2 and grows without bound as the attention weight shrinks. Such a route is therefore governed by the strongest read of $s$ by the answer position, but the layer of the jump is chosen by the cost of the whole route, and $\mathrm{RIFT}(s)\ge1/\min_\ell R_\ell(s)$, with equality when the best route jumps once. The strongest attention weight $\max_\ell\bar A^{\ell}_{ts}$, a diagnostic in Section~\ref{sec:closest}, keeps only the middle term.

\subsection{All-path aggregates}
The \emph{effective resistance} $R_{\mathrm{eff}}(s\to t)$ is the minimum energy $\sum_e f_e^2/c_e$ of a unit flow from $(0,s)$ to $(L,t)$ (Thomson's principle~\cite{doyle1984random}). We use its directed variant, in which flows must be nonnegative along the causal direction, and score $1/R_{\mathrm{eff}}$. \emph{Attention flow} is the maximum flow from $(0,s)$ to $(L,t)$ with capacities $c$~\cite{abnar2020flow}. \emph{Attention rollout} is the entry $(t,s)$ of $W_L\cdots W_1$ with $W_\ell=[c_\ell(i,j)]$, the sum over all paths of the product of their weights~\cite{abnar2020flow}. \emph{Token distance} scores $1/(t-s)$. All scores use the same clean forward pass and never see the answer.

\section{Why All-Path Aggregates Favor Early Tokens}
\label{sec:theory}

In an electric circuit, adding parallel wires lowers the resistance. In a causal network, a route from a source $s$ to the answer $t$ chooses non-decreasing positions between $s$ and $t$ across the $L$ layers, so there are $N(s,t)=\binom{t-s+L-1}{L-1}$ routes~\cite{barbero2024glasses}. For $L=28$ this gives 28 routes from the previous token and about $3\times10^{27}$ from a token 100 positions back. All-path aggregates reward the number of routes, or the number of routes that can carry flow independently. The following proposition shows that, when all edges are equally strong, both grow with distance from the answer.

\begin{proposition}[Equal conductances]
\label{prop:equal}
Let every edge have the same conductance $c$, and let $L\ge2$. For a source $s\le t$:
(a) every path has resistance $L/c$, so $\mathrm{RIFT}(s)=c/L$ for all sources;
(b) the maximum flow from $(0,s)$ to $(L,t)$ equals $c\,(t-s+1)$;
(c) $c\,(t-s+1)/L \le 1/R_{\mathrm{eff}}(s\to t)\le c\,(t-s+1)$, and $1/R_{\mathrm{eff}}$ never decreases as $s$ moves away from $t$;
(d) attention rollout equals $c^L N(s,t)$, which grows like $(t-s)^{L-1}$.
\end{proposition}
\begin{IEEEproof}[Proof sketch]
(a) Every path has exactly $L$ edges. (b) The $t-s+1$ edges leaving $(0,s)$ towards positions that can still reach $t$ form a cut of capacity $c\,(t-s+1)$. Because $L\ge2$, the $t-s+1$ routes that move to position $j\in[s,t]$ in the first layer, stay at $j$, and move to $t$ in the last layer are edge-disjoint; they saturate this cut. (c) Sending $1/(t-s+1)$ units along each of these routes costs energy $L/(c\,(t-s+1))$, which upper-bounds $R_{\mathrm{eff}}$. By the Cauchy--Schwarz inequality, any unit flow through the cut costs at least $1/(c\,(t-s+1))$. Any unit flow from $(0,s)$ can be sent from $(0,s-1)$ into the same first-layer nodes at the same energy, so an earlier source never has a higher $R_{\mathrm{eff}}$. (d) Rollout sums the weight $c^L$ of every path.
\end{IEEEproof}

With equal conductances, then, attention flow grows linearly with distance from the answer, effective conductance stays within a factor $L$ of the same linear function, rollout grows polynomially with degree $L-1$, and \rift{} is constant. Part (c) makes precise the conjecture that earlier tokens have lower effective resistance~\cite{lee2025gnnlens}. Equal conductances are a structural null rather than a model of row-normalized attention, but they capture what dominates in practice. Replacing every conductance of the 1.7B models by a constant reproduces the ranking of effective resistance on their Jacobian graphs almost exactly (Section~\ref{sec:position}). For attention flow on row-normalized graphs, the cut into the answer node bounds the flow by $\sum_{j=s}^{t}c_L(t,j)$, which shrinks as the source moves toward the answer; for the 1.7B instruction-tuned model the flow equals this bound exactly. For row-normalized attention, products of causal attention matrices concentrate on early tokens as depth grows~\cite{wu2025positionbias}, and residual connections temper this into broad, often U-shaped profiles~\cite{herasimchyk2026structural,chowdhury2026birth}.

Why should a single route reflect influence? We offer a hypothesis, not a proof. With attention held fixed, influence would be linear and equal to the sum over routes of the products of their weights, which is rollout; in that sum, the many weak routes from early tokens add up to a position effect (Proposition~\ref{prop:equal}(d)). Our results suggest that content instead travels through a few strong routes, consistent with dedicated copying and retrieval heads~\cite{olsson2022induction,wu2024retrievalhead}. Series resistance penalizes a weak link sharply (a link of weight $0.01$ adds $100$), so \rift{} is governed by the bottlenecks of the best route. A multiplicative path score ($-\log c$ per edge) penalizes weak links less and predicts influence less well (Section~\ref{sec:closest}).

\section{Experimental Setup}

\textbf{Models.} We use four decoders from two families: three Qwen3 models~\cite{qwen3report}, namely Qwen3-1.7B-Base, its instruction-tuned version Qwen3-1.7B (both 28 layers) and Qwen3-4B-Instruct-2507 (36 layers), and Llama-3.2-3B-Instruct~\cite{grattafiori2024llama3} (28 layers). All four are used as released, without fine-tuning; inference runs in FP32 with explicit attention probabilities.

\textbf{Synthetic tasks.} We use three controlled in-context retrieval families whose answer is a single letter. In \emph{single-hop} lookup the prompt lists records such as ``\texttt{river: M}'' and asks for the value of one key. In \emph{two-hop} lookup a key maps to an intermediate word that maps to a letter, with two competing chains. In \emph{first-value} recall each key appears twice with conflicting values and the model must return the first one. Prompts have about 130 or 260 tokens (\val{C1-len-min}--\val{C1-len-max} in the confirmation data). Background prose is either interleaved between the records or placed as one block before them: the \emph{mixed} layout interleaves it for single-hop prompts and uses a block for the other two families, and the \emph{interleaved} layout interleaves it for all three. A task and length are used for a model only if it answers at least 80\% of 20 separate screening prompts; afterwards, every prompt is scored whether or not the model answers it correctly. Llama qualified for single-hop lookup and for first-value recall at the shorter length only.

\textbf{Causal labels.} For each prompt we score 12 source tokens: the word tokens of the records that lead to the answer (one record, or both records of a two-hop chain), up to six word tokens of other records, and filler words. For a source token, we replace it by a token of the same class (an unused letter, an unused key word, or another filler word) and record the change in the answer margin $z_{\mathrm{answer}}-z_{\mathrm{contrast}}$, where the contrast letter never appears as a value. We do this twice, with two different randomly drawn replacements, and use the mean absolute change as the label. The Spearman agreement between the two replacements (\emph{reliability}) shows how noisy single replacements are. Secondary labels are the Jacobian sensitivity of the final state, the norm of the margin gradient, and the signed change.

\textbf{Natural text.} We also use extractive question answering on Wikipedia paragraphs from the SQuAD v1.1 development set~\cite{rajpurkar2016squad}. We keep questions whose answer is one word that occurs once in its paragraph and is a single token for both tokenizers, at most one question per paragraph (\val{N-pool} candidates), and ask for one word copied from the passage. The two instruction-tuned models, Qwen3-1.7B and Llama-3.2-3B-Instruct, are each scored on the first 64 questions, in a fixed random order, that they answer correctly (\val{H-cleanacc}\% and \val{I-cleanacc}\% of the questions scanned). The 12 sources are the answer word, up to three words of its sentence that also occur in the question, up to two such words in other sentences, and words drawn at random across the paragraph. Each source is replaced twice by a random word with the same capitalization; the label uses the change in the log-probability of the answer instead of the margin, and the signed change is the only secondary label.

\textbf{Evaluation.} For each prompt we compute the Spearman correlation between a score and the labels of its 12 sources, average it over prompts, and compare methods with paired differences. 95\% confidence intervals are percentile intervals from 10{,}000 bootstrap resamples of prompts, stratified by task and length (unstratified for SQuAD). They reflect prompt sampling only; datasets that share a model are not independent. To measure position dependence, we also correlate each score with token proximity, $1/(t-s)$; negative values mean that earlier tokens score higher.

\textbf{Protocol.} An \emph{exploration} phase (\val{Mb-n} prompts for each 1.7B model) evaluated effective resistance, the measure we initially hypothesized, and identified the path of least resistance as a far better predictor. A \emph{confirmation} phase then tested \rift{} on five held-out Qwen3 datasets of 64 prompts (Table~\ref{tab:main}) with seeds disjoint from exploration. Datasets A and B repeat the exploration conditions on new prompts; C to E add a larger model and the interleaved layout. Code, configuration, sample sizes and the decision rule were fixed and hashed before any confirmation label was computed, and labels were scored once. The rule required \rift{} to beat each of three baselines with a 95\% interval above zero in every dataset. We fixed these baselines in advance as the standard scores that, like \rift{}, need nothing beyond one forward pass (no gradients, extra forward passes or training): token distance, attention rollout and attention flow. Gradient-, perturbation- and training-based attribution~\cite{ancona2018gradients,zhang2024patching,pramanik2026heta,cohenwang2025at2} belongs to a different cost regime, and contribution-aware graphs~\cite{kobayashi2020norms,ferrando2022alti} change the conductances rather than how routes are aggregated; we report gradient$\times$input and value-norm conductances as diagnostics (Section~\ref{sec:robust}). We then repeated the confirmation on a second model family, Llama-3.2-3B-Instruct (datasets F and G, one per layout, 64 fresh prompts each), and on natural text (datasets H and I), each time under the same rule, fixed before any label was computed. Everything ran on two RTX 5080 GPUs (\val{M1-gpuh} GPU-hours for exploration and \val{C1-gpuh}, \val{C2-gpuh} and \val{C3-gpuh} for the three confirmations).

\begin{table*}[t]
\caption{Predicting causal token influence: mean within-prompt Spearman correlation with the replacement label (the mean of two replacements), and the paired advantage of \rift{} over each baseline with 95\% bootstrap intervals. Acc.: share of the scored prompts answered correctly. Rel.: agreement between two single replacements; \rift{}$_1$: \rift{} against a single replacement, the like-for-like counterpart of Rel. A--E and H use Qwen3, F, G and I Llama-3.2-3B-Instruct; H and I are correctly answered SQuAD questions, with the log-probability label. In the Qwen3 datasets, attention flow ranks tokens from earliest to latest (exactly in all but one prompt), so its correlation with the label is the negative of that of distance.}
\label{tab:main}
\centering
\setlength{\tabcolsep}{2.6pt}
\begin{tabular}{@{}lllcrrccccc@{}}
\toprule
& & & & & & & & \multicolumn{3}{c}{Advantage of \rift{} over}\\
\cmidrule(l){9-11}
& Model & Prompts per task & Layout & Acc.\,\% & Rel. & \rift{}$_1$ & \rift{} & distance & rollout & attention flow\\
\midrule
\multicolumn{11}{@{}l}{\emph{Synthetic tasks: confirmation (held out, pre-specified)}}\\
A & 1.7B base & \val{A-comp} & mixed & \val{A-acc} & \val{A-ceil} & \val{A-riftsingle} & \textbf{\val{A-plr}} & \val{A-ddist} [\val{A-ddist-lo}, \val{A-ddist-hi}] & \val{A-droll} [\val{A-droll-lo}, \val{A-droll-hi}] & \val{A-dflow} [\val{A-dflow-lo}, \val{A-dflow-hi}]\\
B & 1.7B instruct & \val{B-comp} & mixed & \val{B-acc} & \val{B-ceil} & \val{B-riftsingle} & \textbf{\val{B-plr}} & \val{B-ddist} [\val{B-ddist-lo}, \val{B-ddist-hi}] & \val{B-droll} [\val{B-droll-lo}, \val{B-droll-hi}] & \val{B-dflow} [\val{B-dflow-lo}, \val{B-dflow-hi}]\\
C & 4B instruct & \val{C-comp} & mixed & \val{C-acc} & \val{C-ceil} & \val{C-riftsingle} & \textbf{\val{C-plr}} & \val{C-ddist} [\val{C-ddist-lo}, \val{C-ddist-hi}] & \val{C-droll} [\val{C-droll-lo}, \val{C-droll-hi}] & \val{C-dflow} [\val{C-dflow-lo}, \val{C-dflow-hi}]\\
D & 1.7B base & \val{D-comp} & interleaved & \val{D-acc} & \val{D-ceil} & \val{D-riftsingle} & \textbf{\val{D-plr}} & \val{D-ddist} [\val{D-ddist-lo}, \val{D-ddist-hi}] & \val{D-droll} [\val{D-droll-lo}, \val{D-droll-hi}] & \val{D-dflow} [\val{D-dflow-lo}, \val{D-dflow-hi}]\\
E & 1.7B instruct & \val{E-comp} & interleaved & \val{E-acc} & \val{E-ceil} & \val{E-riftsingle} & \textbf{\val{E-plr}} & \val{E-ddist} [\val{E-ddist-lo}, \val{E-ddist-hi}] & \val{E-droll} [\val{E-droll-lo}, \val{E-droll-hi}] & \val{E-dflow} [\val{E-dflow-lo}, \val{E-dflow-hi}]\\
F & Llama 3B & \val{F-comp} & mixed & \val{F-acc} & \val{F-ceil} & \val{F-riftsingle} & \textbf{\val{F-plr}} & \val{F-ddist} [\val{F-ddist-lo}, \val{F-ddist-hi}] & \val{F-droll} [\val{F-droll-lo}, \val{F-droll-hi}] & \val{F-dflow} [\val{F-dflow-lo}, \val{F-dflow-hi}]\\
G & Llama 3B & \val{G-comp} & interleaved & \val{G-acc} & \val{G-ceil} & \val{G-riftsingle} & \textbf{\val{G-plr}} & \val{G-ddist} [\val{G-ddist-lo}, \val{G-ddist-hi}] & \val{G-droll} [\val{G-droll-lo}, \val{G-droll-hi}] & \val{G-dflow} [\val{G-dflow-lo}, \val{G-dflow-hi}]\\
\midrule
\multicolumn{11}{@{}l}{\emph{Natural text (SQuAD): confirmation (held out, pre-specified)}}\\
H & 1.7B instruct & \val{H-comp} & Wikipedia & \val{H-acc} & \val{H-ceil} & \val{H-riftsingle} & \textbf{\val{H-plr}} & \val{H-ddist} [\val{H-ddist-lo}, \val{H-ddist-hi}] & \val{H-droll} [\val{H-droll-lo}, \val{H-droll-hi}] & \val{H-dflow} [\val{H-dflow-lo}, \val{H-dflow-hi}]\\
I & Llama 3B & \val{I-comp} & Wikipedia & \val{I-acc} & \val{I-ceil} & \val{I-riftsingle} & \textbf{\val{I-plr}} & \val{I-ddist} [\val{I-ddist-lo}, \val{I-ddist-hi}] & \val{I-droll} [\val{I-droll-lo}, \val{I-droll-hi}] & \val{I-dflow} [\val{I-dflow-lo}, \val{I-dflow-hi}]\\
\midrule
\multicolumn{11}{@{}l}{\emph{Exploration (\val{Mb-n} prompts per model)}}\\
& 1.7B base & single, first & mixed & \val{Mb-acc} & \val{Mb-ceil} & -- & \val{Mb-plr} & \val{Mb-ddist} [\val{Mb-ddist-lo}, \val{Mb-ddist-hi}] & \val{Mb-droll} [\val{Mb-droll-lo}, \val{Mb-droll-hi}] & \val{Mb-dflow} [\val{Mb-dflow-lo}, \val{Mb-dflow-hi}]\\
& 1.7B instruct & single, two-hop & mixed & \val{Mi-acc} & \val{Mi-ceil} & -- & \val{Mi-plr} & \val{Mi-ddist} [\val{Mi-ddist-lo}, \val{Mi-ddist-hi}] & \val{Mi-droll} [\val{Mi-droll-lo}, \val{Mi-droll-hi}] & \val{Mi-dflow} [\val{Mi-dflow-lo}, \val{Mi-dflow-hi}]\\
\bottomrule
\end{tabular}
\end{table*}

\section{Results}

\subsection{\rift{} predicts causal influence}
\label{sec:main}
Table~\ref{tab:main} and Fig.~\ref{fig:main} give the confirmatory results. \rift{} correlates with the replacement labels at \val{C1-plr-min}--\val{C1-plr-max}. Single replacements are noisy: two different random replacements of the same token agree at only \val{C1-rel12-min}--\val{C1-rel12-max}. \rift{} predicts one replacement at \val{C1-riftsingle-min}--\val{C1-riftsingle-max}, i.e., \val{C1-lfl-min}--\val{C1-lfl-max}\% as well as the other replacement does. A Spearman--Brown estimate, only a heuristic here because it treats the two replacements as parallel measurements and our correlations are rank-based, puts the attainable correlation with the averaged label near $\sqrt{2r/(1+r)}=\val{C1-sbceil-min}$--$\val{C1-sbceil-max}$, where $r$ is the agreement of single replacements; headroom remains. All \val{n-comp} pre-specified comparisons favor \rift{}, with the lowest interval bound at \val{C1-lowest-ci}. In the Qwen3 datasets, attention flow ranks the tokens from earliest to latest (exactly in \val{Q-flow-exact} of \val{Q-flow-n} prompts), so there its comparisons nearly duplicate those against token distance. The strongest baseline in each dataset trails by \val{C1-bestgapu-min} to \val{C1-bestgapu-max}, and no baseline exceeds \val{C1-tools-max}. The result holds for the larger 4B model (\val{C-plr}), for interleaved prose (\val{D-plr} and \val{E-plr}), for a second model family (Llama: \val{F-plr} and \val{G-plr}), for natural text (\val{H-plr} and \val{I-plr}; Section~\ref{sec:natural}), and in exploration (\val{Mb-plr} and \val{Mi-plr}). The widest path, the other single-route score (a diagnostic; Section~\ref{sec:closest}), comes close at \val{C1-widest-min}--\val{C1-widest-max} (Fig.~\ref{fig:main}); no all-path aggregate does. The labels themselves are nearly independent of position: per dataset, token distance correlates with them between \val{C1-distlab-min} and \val{C1-distlab-max}. \rift{} also leads within every task family (Table~\ref{tab:tasks}). Token distance is competitive only for two-hop prompts, where the prose block places the relevant records right before the question. Reported agreement between rollout and patching varies across settings (e.g., about 0.5 in~\cite{pramanik2026heta}); in our prompts, where the relevant record can sit anywhere in the context, rollout's position bias dominates.

\begin{table}[t]
\caption{Confirmation results by task family (prompt-weighted means over the \val{n-conf-word} datasets; $N$ is the number of prompts). Rel.: agreement between two single replacements; the score columns use their mean.}
\label{tab:tasks}
\centering
\setlength{\tabcolsep}{4pt}
\begin{tabular}{@{}lrrcccc@{}}
\toprule
Task & $N$ & Rel. & \rift{} & Distance & Rollout & Attn.\ flow\\
\midrule
Single-hop & \val{task-single-n} & \val{task-single-ceil} & \textbf{\val{task-single-plr}} & \val{task-single-dist} & \val{task-single-roll} & \val{task-single-flow}\\
Two-hop & \val{task-twohop-n} & \val{task-twohop-ceil} & \textbf{\val{task-twohop-plr}} & \val{task-twohop-dist} & \val{task-twohop-roll} & \val{task-twohop-flow}\\
First-value & \val{task-first-n} & \val{task-first-ceil} & \textbf{\val{task-first-plr}} & \val{task-first-dist} & \val{task-first-roll} & \val{task-first-flow}\\
SQuAD & \val{task-squad-n} & \val{task-squad-ceil} & \textbf{\val{task-squad-plr}} & \val{task-squad-dist} & \val{task-squad-roll} & \val{task-squad-flow}\\
\bottomrule
\end{tabular}
\end{table}

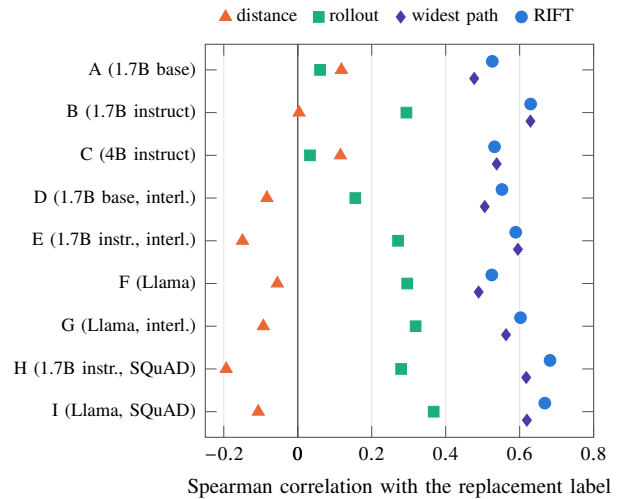
\begin{figure}[t]
\centering
\begin{tikzpicture}
\begin{axis}[
  scale only axis, width=0.58\columnwidth, height=5.2cm,
  xmin=-0.25, xmax=0.8,
  ymin=0.4, ymax=9.6,
  ytick={1,2,3,4,5,6,7,8,9},
  yticklabels={{I (Llama, SQuAD)},{H (1.7B instr., SQuAD)},{G (Llama, interl.)},{F (Llama)},{E (1.7B instr., interl.)},{D (1.7B base, interl.)},{C (4B instruct)},{B (1.7B instruct)},{A (1.7B base)}},
  yticklabel style={font=\scriptsize},
  xticklabel style={font=\scriptsize},
  xlabel={\footnotesize Spearman correlation with the replacement label},
  axis line style={inkmuted,line width=0.4pt},
  tick style={inkmuted},
  xmajorgrids, grid style={gridgray,line width=0.3pt},
  extra x ticks={0}, extra x tick style={grid style={inkmuted,line width=0.5pt}},
  legend style={font=\scriptsize,draw=none,fill=none,at={(0.5,1.02)},anchor=south,legend columns=4,/tikz/every even column/.append style={column sep=4pt}},
  clip=false,
]
\addplot[only marks,mark=triangle*,mark size=2.6pt,color=slotorange] table[x=dist,y=y] {data_main.dat};
\addlegendentry{distance}
\addplot[only marks,mark=square*,mark size=2.0pt,color=slotaqua] table[x=roll,y=y] {data_main.dat};
\addlegendentry{rollout}
\addplot[only marks,mark=diamond*,mark size=2.9pt,color=slotviolet,mark options={draw=white,line width=0.4pt,fill=slotviolet}] table[x=widest,y expr=\thisrow{y}-0.2] {data_main.dat};
\addlegendentry{widest path}
\addplot[only marks,mark=*,mark size=2.6pt,color=slotblue,mark options={draw=white,line width=0.4pt,fill=slotblue}] table[x=plr,y expr=\thisrow{y}+0.2] {data_main.dat};
\addlegendentry{\rift{}}
\end{axis}
\end{tikzpicture}
\caption{Confirmatory results for the \val{n-conf-word} held-out datasets (A--E and H: Qwen3; F, G and I: Llama-3.2-3B-Instruct; H and I: SQuAD). \rift{} (blue) is far above token distance and rollout in every dataset; the widest path (violet), its single-route sibling and a diagnostic outside the decision rule, comes close. Attention flow is omitted (for Qwen3 it is the mirror image of distance); Table~\ref{tab:main} gives all values with the label reliability.}
\label{fig:main}
\end{figure}

\subsection{Natural text}
\label{sec:natural}
On SQuAD paragraphs (datasets H and I), \rift{} reaches \val{H-plr} and \val{I-plr}, its highest correlations in this study, while no baseline exceeds \val{N-tools-max}. The labels are no noisier than in the synthetic tasks: two replacements agree at \val{H-rel12} and \val{I-rel12}, and \rift{} predicts a single replacement \val{H-lfl}\% and \val{I-lfl}\% as well as the other replacement does. \rift{} ranks the answer word first among the 12 sources in \val{H-top-plr} and \val{I-top-plr} of the 64 prompts; token distance, rollout and attention flow do so in at most \val{N-top-tools-max}. \rift{} also exceeds an indicator that marks the answer word and its question words (\val{H-oracle} and \val{I-oracle}), so it captures more than the location of the answer: among the remaining words, whose effects are small and noisy (two single replacements agree at \val{H-relnonrel} and \val{I-relnonrel}), it still correlates with the effects at \val{H-riftnonrel} and \val{I-riftnonrel}. Rollout stays strongly position-driven ($\rho=\val{H-pos-roll}$ and \val{I-pos-roll} with inverse distance). Fig.~\ref{fig:squad} shows a representative question, chosen by rule as the prompt whose \rift{} correlation is closest to the dataset median ($\rho=\val{exs-rho}$, median \val{exs-median}). \rift{} ranks the answer word first and the word with the second-largest effect second, as the labels do, while distance favors the last words of the paragraph and attention flow orders the words from earliest to latest.

\begin{figure}[t]
\centering
\begin{tikzpicture}[x=0.56cm,y=0.42cm]
\node[anchor=east,font=\scriptsize] at (-0.1,0) {Label};
\fill[slotblue!50!white] (0.04,-0.42) rectangle (0.96,0.42);
\fill[slotblue!75!white] (1.04,-0.42) rectangle (1.96,0.42);
\fill[slotblue!8!white] (2.04,-0.42) rectangle (2.96,0.42);
\fill[slotblue!67!white] (3.04,-0.42) rectangle (3.96,0.42);
\fill[slotblue!83!white] (4.04,-0.42) rectangle (4.96,0.42);
\fill[slotblue!58!white] (5.04,-0.42) rectangle (5.96,0.42);
\fill[slotblue!92!white] (6.04,-0.42) rectangle (6.96,0.42);
\fill[slotblue!100!white] (7.04,-0.42) rectangle (7.96,0.42);
\fill[slotblue!25!white] (8.04,-0.42) rectangle (8.96,0.42);
\fill[slotblue!33!white] (9.04,-0.42) rectangle (9.96,0.42);
\fill[slotblue!16!white] (10.04,-0.42) rectangle (10.96,0.42);
\fill[slotblue!41!white] (11.04,-0.42) rectangle (11.96,0.42);
\node[anchor=east,font=\scriptsize] at (-0.1,-1) {RIFT};
\fill[slotblue!58!white] (0.04,-1.42) rectangle (0.96,-0.5800000000000001);
\fill[slotblue!83!white] (1.04,-1.42) rectangle (1.96,-0.5800000000000001);
\fill[slotblue!33!white] (2.04,-1.42) rectangle (2.96,-0.5800000000000001);
\fill[slotblue!50!white] (3.04,-1.42) rectangle (3.96,-0.5800000000000001);
\fill[slotblue!41!white] (4.04,-1.42) rectangle (4.96,-0.5800000000000001);
\fill[slotblue!75!white] (5.04,-1.42) rectangle (5.96,-0.5800000000000001);
\fill[slotblue!92!white] (6.04,-1.42) rectangle (6.96,-0.5800000000000001);
\fill[slotblue!100!white] (7.04,-1.42) rectangle (7.96,-0.5800000000000001);
\fill[slotblue!67!white] (8.04,-1.42) rectangle (8.96,-0.5800000000000001);
\fill[slotblue!25!white] (9.04,-1.42) rectangle (9.96,-0.5800000000000001);
\fill[slotblue!16!white] (10.04,-1.42) rectangle (10.96,-0.5800000000000001);
\fill[slotblue!8!white] (11.04,-1.42) rectangle (11.96,-0.5800000000000001);
\node[anchor=east,font=\scriptsize] at (-0.1,-2) {Distance};
\fill[slotblue!8!white] (0.04,-2.42) rectangle (0.96,-1.58);
\fill[slotblue!16!white] (1.04,-2.42) rectangle (1.96,-1.58);
\fill[slotblue!25!white] (2.04,-2.42) rectangle (2.96,-1.58);
\fill[slotblue!33!white] (3.04,-2.42) rectangle (3.96,-1.58);
\fill[slotblue!41!white] (4.04,-2.42) rectangle (4.96,-1.58);
\fill[slotblue!50!white] (5.04,-2.42) rectangle (5.96,-1.58);
\fill[slotblue!58!white] (6.04,-2.42) rectangle (6.96,-1.58);
\fill[slotblue!67!white] (7.04,-2.42) rectangle (7.96,-1.58);
\fill[slotblue!75!white] (8.04,-2.42) rectangle (8.96,-1.58);
\fill[slotblue!83!white] (9.04,-2.42) rectangle (9.96,-1.58);
\fill[slotblue!92!white] (10.04,-2.42) rectangle (10.96,-1.58);
\fill[slotblue!100!white] (11.04,-2.42) rectangle (11.96,-1.58);
\node[anchor=east,font=\scriptsize] at (-0.1,-3) {Rollout};
\fill[slotblue!100!white] (0.04,-3.42) rectangle (0.96,-2.58);
\fill[slotblue!92!white] (1.04,-3.42) rectangle (1.96,-2.58);
\fill[slotblue!83!white] (2.04,-3.42) rectangle (2.96,-2.58);
\fill[slotblue!67!white] (3.04,-3.42) rectangle (3.96,-2.58);
\fill[slotblue!75!white] (4.04,-3.42) rectangle (4.96,-2.58);
\fill[slotblue!58!white] (5.04,-3.42) rectangle (5.96,-2.58);
\fill[slotblue!50!white] (6.04,-3.42) rectangle (6.96,-2.58);
\fill[slotblue!41!white] (7.04,-3.42) rectangle (7.96,-2.58);
\fill[slotblue!25!white] (8.04,-3.42) rectangle (8.96,-2.58);
\fill[slotblue!16!white] (9.04,-3.42) rectangle (9.96,-2.58);
\fill[slotblue!8!white] (10.04,-3.42) rectangle (10.96,-2.58);
\fill[slotblue!33!white] (11.04,-3.42) rectangle (11.96,-2.58);
\node[anchor=east,font=\scriptsize] at (-0.1,-4) {Attn.\ flow};
\fill[slotblue!100!white] (0.04,-4.42) rectangle (0.96,-3.58);
\fill[slotblue!92!white] (1.04,-4.42) rectangle (1.96,-3.58);
\fill[slotblue!83!white] (2.04,-4.42) rectangle (2.96,-3.58);
\fill[slotblue!75!white] (3.04,-4.42) rectangle (3.96,-3.58);
\fill[slotblue!67!white] (4.04,-4.42) rectangle (4.96,-3.58);
\fill[slotblue!58!white] (5.04,-4.42) rectangle (5.96,-3.58);
\fill[slotblue!50!white] (6.04,-4.42) rectangle (6.96,-3.58);
\fill[slotblue!41!white] (7.04,-4.42) rectangle (7.96,-3.58);
\fill[slotblue!33!white] (8.04,-4.42) rectangle (8.96,-3.58);
\fill[slotblue!25!white] (9.04,-4.42) rectangle (9.96,-3.58);
\fill[slotblue!16!white] (10.04,-4.42) rectangle (10.96,-3.58);
\fill[slotblue!8!white] (11.04,-4.42) rectangle (11.96,-3.58);
\node[anchor=east,font=\scriptsize] at (-0.1,-5) {Eff.\ resist.};
\fill[slotblue!100!white] (0.04,-5.42) rectangle (0.96,-4.58);
\fill[slotblue!67!white] (1.04,-5.42) rectangle (1.96,-4.58);
\fill[slotblue!83!white] (2.04,-5.42) rectangle (2.96,-4.58);
\fill[slotblue!92!white] (3.04,-5.42) rectangle (3.96,-4.58);
\fill[slotblue!75!white] (4.04,-5.42) rectangle (4.96,-4.58);
\fill[slotblue!58!white] (5.04,-5.42) rectangle (5.96,-4.58);
\fill[slotblue!50!white] (6.04,-5.42) rectangle (6.96,-4.58);
\fill[slotblue!25!white] (7.04,-5.42) rectangle (7.96,-4.58);
\fill[slotblue!33!white] (8.04,-5.42) rectangle (8.96,-4.58);
\fill[slotblue!41!white] (9.04,-5.42) rectangle (9.96,-4.58);
\fill[slotblue!8!white] (10.04,-5.42) rectangle (10.96,-4.58);
\fill[slotblue!16!white] (11.04,-5.42) rectangle (11.96,-4.58);
\node[anchor=north east,rotate=50,font=\scriptsize] at (0.6,-5.55) {a\textsubscript{26}};
\node[anchor=north east,rotate=50,font=\scriptsize] at (1.6,-5.55) {boat\textsubscript{29}};
\node[anchor=north east,rotate=50,font=\scriptsize] at (2.6,-5.55) {the\textsubscript{54}};
\node[anchor=north east,rotate=50,font=\scriptsize] at (3.6,-5.55) {the\textsubscript{60}};
\node[anchor=north east,rotate=50,font=\scriptsize] at (4.6,-5.55) {workings\textsubscript{61}};
\node[anchor=north east,rotate=50,font=\scriptsize\bfseries] at (5.6,-5.55) {boat\textsubscript{64}};
\node[anchor=north east,rotate=50,font=\scriptsize] at (6.6,-5.55) {trained\textsubscript{80}};
\node[anchor=north east,rotate=50,font=\scriptsize\bfseries] at (7.6,-5.55) {monkey\textsubscript{81}};
\node[anchor=north east,rotate=50,font=\scriptsize] at (8.6,-5.55) {type\textsubscript{99}};
\node[anchor=north east,rotate=50,font=\scriptsize] at (9.6,-5.55) {of\textsubscript{127}};
\node[anchor=north east,rotate=50,font=\scriptsize] at (10.6,-5.55) {programs\textsubscript{133}};
\node[anchor=north east,rotate=50,font=\scriptsize] at (11.6,-5.55) {Tesla\textsubscript{135}};
\end{tikzpicture}
\caption{A representative SQuAD question for Qwen3-1.7B (dataset H; \val{exs-tokens} tokens): ``\val{exs-question}'' (answer: \emph{\val{exs-answer}}). Columns are the 12 scored words in position order (subscript: position; the answer word and its question words in bold); darker cells mark tokens ranked as more influential within the prompt.}
\label{fig:squad}
\end{figure}
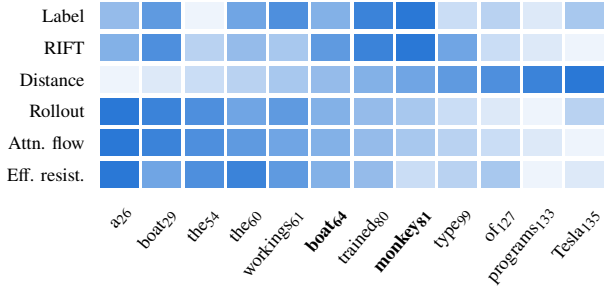

\subsection{Accuracy and cost}
\label{sec:cost}
Fig.~\ref{fig:cost} compares accuracy and cost on the \val{cost-n} SQuAD prompts. To compare like with like, accuracy is here the Spearman correlation with the effect of a single replacement, which a second, different replacement predicts at \val{cost-acc-rep}. Measuring such effects for every token takes one forward pass per token: \val{cost-rep-s-H}\,s per prompt for Qwen3-1.7B and \val{cost-rep-s-I}\,s for Llama in our unbatched FP32 runs on an RTX 5080 GPU. \rift{} reaches \val{cost-acc-rift}, \val{cost-ratio}\% of that agreement, without any extra forward pass: one sweep of Algorithm~\ref{alg:plr} over the attention weights scores every token in \val{cost-rift-ms}\,ms (median; \val{cost-rift-min-ms}--\val{cost-rift-max-ms}\,ms across all \val{n-conf-word} datasets) on one thread of an AMD Ryzen 9 9900X CPU, \val{speedup-rep-min} to \val{speedup-rep-max} times faster than our unbatched replacement runs, and its scores match the per-source computation to machine precision. Rollout is faster still (\val{cost-roll-ms}\,ms) but far less accurate (\val{cost-acc-roll}). Attention flow and effective resistance must be solved once per source, which takes about \val{cost-flow-s} and \val{cost-res-s}\,s per prompt with our solvers (per-source time multiplied by the number of tokens); pruning can speed up flow~\cite{xu2025flowprune}, but both scores are barely related to the labels (\val{cost-acc-flow} and \val{cost-acc-res}).

\begin{figure}[t]
\centering
\begin{tikzpicture}
\begin{axis}[
  width=\columnwidth, height=4.3cm, xmode=log, xmin=1e-6, xmax=1e4, ymin=-0.25, ymax=0.85,
  xtick={1e-6,1e-4,1e-2,1,1e2,1e4},
  xlabel={\footnotesize Time to score every token of a prompt (s)},
  ylabel={\footnotesize Spearman with one replacement},
  xticklabel style={font=\scriptsize}, yticklabel style={font=\scriptsize},
  axis line style={inkmuted,line width=0.4pt}, tick style={inkmuted},
  ymajorgrids, grid style={gridgray,line width=0.3pt},
  extra y ticks={0}, extra y tick style={grid style={inkmuted,line width=0.5pt}},
  clip=false]
\addplot[only marks,mark=triangle*,mark size=2.8pt,color=slotorange] coordinates {(2.92601e-06,-0.1241)};
\node[right,font=\scriptsize,inner sep=2.5pt] at (axis cs:2.92601e-06,-0.1241) {distance};
\addplot[only marks,mark=square*,mark size=2.2pt,color=slotaqua] coordinates {(9.80095e-05,0.2882)};
\node[right,font=\scriptsize,inner sep=2.5pt] at (axis cs:9.80095e-05,0.2882) {rollout};
\addplot[only marks,mark=*,mark size=3.0pt,color=slotblue] coordinates {(0.00134205,0.6161)};
\node[left,font=\scriptsize,inner sep=2.5pt] at (axis cs:0.00134205,0.6161) {\textbf{RIFT}};
\addplot[only marks,mark=star,mark size=3.6pt,color=black,line width=0.9pt] coordinates {(7.60822,0.6532)};
\node[above,font=\scriptsize,inner sep=2.5pt] at (axis cs:7.60822,0.6532) {replacement experiment};
\addplot[only marks,mark=pentagon*,mark size=2.6pt,color=slotmagenta] coordinates {(55.5759,0.1122)};
\node[below left,font=\scriptsize,inner sep=2.5pt] at (axis cs:55.5759,0.1122) {effective resistance};
\addplot[only marks,mark=diamond,mark size=3.4pt,color=slotyellow,line width=0.9pt] coordinates {(73.0974,0.1197)};
\node[above left,font=\scriptsize,inner sep=2.5pt] at (axis cs:73.0974,0.1197) {attention flow};
\end{axis}
\end{tikzpicture}
\caption{Accuracy versus cost on the SQuAD prompts (datasets H and I). Accuracy: Spearman correlation with the effect of one replacement (for the replacement experiment, the agreement between two replacements). Cost: median time to score every token of a prompt. \rift{}, rollout and distance are timed on one CPU thread, attention flow and effective resistance from their per-source solver times, and the replacement experiment from its per-pass GPU time (one forward pass per token).}
\label{fig:cost}
\end{figure}
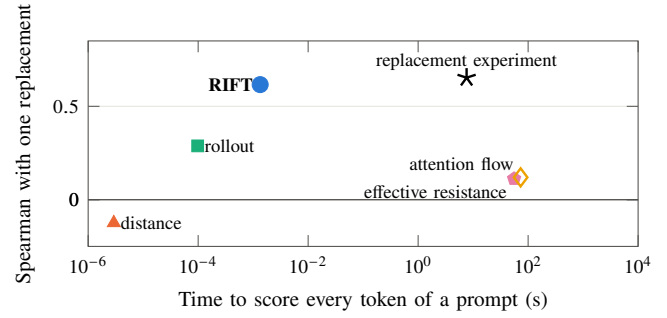

\subsection{Comparison with other attention-based scores}
\label{sec:closest}
Because the best route is usually a single jump (Section~\ref{sec:route}), one may ask whether a simpler read of the attention graph suffices. Table~\ref{tab:closest} compares \rift{} with the strongest attention from the answer position to a token (maximum over layers), the widest path, and a multiplicative path score that sums $-\log c$ along a route (a maximum-probability path, as in~\cite{trevino2024mla}). These scores were added after confirmation, on the same records; they are diagnostics, not part of the decision rule. \rift{} is the best or tied-best score in every dataset. The widest path, its single-route sibling, ties it in \val{C1-wid-tie-of} datasets and trails by \val{C1-wid-siggap-min} to \val{C1-wid-siggap-max} in the others. The strongest single attention weight trails significantly in \val{C1-max-sig-of} datasets but comes close for the Qwen3 base model (A and D), the mean attention from the answer position over layers (not shown) trails significantly in \val{C1-mean-sig-of}, and the multiplicative path in \val{C1-log-sig-of}.

\begin{table}[t]
\caption{Advantage of \rift{} over other attention-based scores (paired difference in mean Spearman with 95\% intervals; computed after confirmation, diagnostic).}
\label{tab:closest}
\centering
\setlength{\tabcolsep}{2.2pt}
\begin{tabular}{@{}lccc@{}}
\toprule
& $-$\,max.\ attention & $-$\,widest path & $-$\,multiplicative path\\
\midrule
A & \val{A-dmax} [\val{A-dmax-lo}, \val{A-dmax-hi}] & \val{A-dwid} [\val{A-dwid-lo}, \val{A-dwid-hi}] & \val{A-dlog} [\val{A-dlog-lo}, \val{A-dlog-hi}]\\
B & \val{B-dmax} [\val{B-dmax-lo}, \val{B-dmax-hi}] & \val{B-dwid} [\val{B-dwid-lo}, \val{B-dwid-hi}] & \val{B-dlog} [\val{B-dlog-lo}, \val{B-dlog-hi}]\\
C & \val{C-dmax} [\val{C-dmax-lo}, \val{C-dmax-hi}] & \val{C-dwid} [\val{C-dwid-lo}, \val{C-dwid-hi}] & \val{C-dlog} [\val{C-dlog-lo}, \val{C-dlog-hi}]\\
D & \val{D-dmax} [\val{D-dmax-lo}, \val{D-dmax-hi}] & \val{D-dwid} [\val{D-dwid-lo}, \val{D-dwid-hi}] & \val{D-dlog} [\val{D-dlog-lo}, \val{D-dlog-hi}]\\
E & \val{E-dmax} [\val{E-dmax-lo}, \val{E-dmax-hi}] & \val{E-dwid} [\val{E-dwid-lo}, \val{E-dwid-hi}] & \val{E-dlog} [\val{E-dlog-lo}, \val{E-dlog-hi}]\\
F & \val{F-dmax} [\val{F-dmax-lo}, \val{F-dmax-hi}] & \val{F-dwid} [\val{F-dwid-lo}, \val{F-dwid-hi}] & \val{F-dlog} [\val{F-dlog-lo}, \val{F-dlog-hi}]\\
G & \val{G-dmax} [\val{G-dmax-lo}, \val{G-dmax-hi}] & \val{G-dwid} [\val{G-dwid-lo}, \val{G-dwid-hi}] & \val{G-dlog} [\val{G-dlog-lo}, \val{G-dlog-hi}]\\
H & \val{H-dmax} [\val{H-dmax-lo}, \val{H-dmax-hi}] & \val{H-dwid} [\val{H-dwid-lo}, \val{H-dwid-hi}] & \val{H-dlog} [\val{H-dlog-lo}, \val{H-dlog-hi}]\\
I & \val{I-dmax} [\val{I-dmax-lo}, \val{I-dmax-hi}] & \val{I-dwid} [\val{I-dwid-lo}, \val{I-dwid-hi}] & \val{I-dlog} [\val{I-dlog-lo}, \val{I-dlog-hi}]\\
\bottomrule
\end{tabular}
\end{table}

\subsection{Finding the tokens that carry the answer}
A practical use of attribution is to point at the tokens that hold the answer. In the synthetic datasets, treating the relevant record tokens as positives, \rift{} separates them from the rest with AUROC \val{C1-auc-plr-min}--\val{C1-auc-plr-max}. Rollout reaches \val{C1-auc-roll-min}--\val{C1-auc-roll-max}, attention flow \val{C1-auc-flow-min}--\val{C1-auc-flow-max}, and token distance stays near chance (\val{C1-auc-dist-min}--\val{C1-auc-dist-max}). Much of the correlation with the labels comes from this separation: an oracle indicator that knows which tokens belong to the relevant records reaches \val{C1-oracle-min}--\val{C1-oracle-max}. Among the remaining tokens, effects are small and noisy (reliability \val{C1-relnonrel-min}--\val{C1-relnonrel-max}), yet \rift{} still correlates with them at \val{C1-riftnonrel-min}--\val{C1-riftnonrel-max}.

\subsection{All-path aggregates track position}
\label{sec:position}
Table~\ref{tab:position} contrasts how strongly each score follows token position with how well it predicts the labels. Attention flow is a perfect function of position in exploration ($\rho=\val{Mb-pos-flow}$ with inverse distance in both models) and in \val{Q-flow-exact} of the \val{Q-flow-n} Qwen3 confirmation prompts; in Llama its correlation with position is \val{L-pos-flow-min} to \val{L-pos-flow-max}. Effective resistance follows position almost as closely, whether its conductances come from attention or from local Jacobians. The structural null of Section~\ref{sec:theory} reproduces the ranking of the Jacobian-based effective resistance with $\rho=\val{Mb-null-rift}$ and $\val{Mi-null-rift}$ while being unrelated to the labels. Rollout is strongly position-driven as well. \rift{} and the widest path are only weakly related to position and are the only scores in the table that predict the labels well. In the confirmation data, effective resistance correlates with the labels at only \val{C1-kirch-min} to \val{C1-kirch-max}.

\begin{table}[t]
\caption{Position dependence versus validity (exploration, \val{Mb-n} prompts per 1.7B model, base\,/\,instruct). Position: Spearman with inverse distance to the answer; label: Spearman with the replacement label. Structural null: effective resistance with every conductance set to the same constant.}
\label{tab:position}
\centering
\setlength{\tabcolsep}{3.5pt}
\begin{tabular}{@{}lcc@{}}
\toprule
Score & Position $\rho$ & Label $\rho$\\
\midrule
Token distance & 1.00\,/\,1.00 & \val{Mb-dist}\,/\,\val{Mi-dist}\\
Attention rollout & \val{Mb-pos-roll}\,/\,\val{Mi-pos-roll} & \val{Mb-roll}\,/\,\val{Mi-roll}\\
Attention flow & \val{Mb-pos-flow}\,/\,\val{Mi-pos-flow} & \val{Mb-flow}\,/\,\val{Mi-flow}\\
Eff.\ resistance, attention & \val{Mb-pos-kirch}\,/\,\val{Mi-pos-kirch} & \val{Mb-kirch}\,/\,\val{Mi-kirch}\\
Eff.\ resistance, Jacobian & \val{Mb-pos-rift}\,/\,\val{Mi-pos-rift} & \val{Mb-kirchjac}\,/\,\val{Mi-kirchjac}\\
Structural null & \val{M1-null-pos}\,/\,\val{Mi-null-pos} & \val{Mb-null-lab}\,/\,\val{Mi-null-lab}\\
\midrule
Widest path & \val{Mb-pos-widest}\,/\,\val{Mi-pos-widest} & \val{Mb-widest}\,/\,\val{Mi-widest}\\
\rift{} (ours) & \val{Mb-pos-plr}\,/\,\val{Mi-pos-plr} & \textbf{\val{Mb-plr}}\,/\,\textbf{\val{Mi-plr}}\\
\bottomrule
\end{tabular}
\end{table}

\subsection{Depth and density}
Fig.~\ref{fig:depth} truncates the attention graph after its first $k$ layers and measures how strongly each score follows position. Attention flow and effective resistance turn position-driven almost immediately (\val{d3-kirch-k2} at $k=2$ and \val{d3-kirch-k4} at $k=4$ for effective resistance) and approach $-1$ with depth. Rollout drifts more slowly, from \val{d3-roll-k2} at $k=2$ to \val{d3-roll-k28} at full depth. \rift{} remains positive at all depths (\val{d3-plr-k28} at $k=28$). Depth is not the whole story. In a 4-layer Transformer that we trained from scratch on similar synthetic retrieval tasks in a pilot study, effective resistance on its Jacobian graph was barely position-driven (\val{toy-jac-pos-kirch}), whereas on its attention graph it already was (\val{toy-att-pos-kirch}). What matters is how widely each token spreads its weight, which determines how many parallel routes carry real weight; depth then multiplies these routes.

\begin{figure}[t]
\centering
\begin{tikzpicture}
\begin{axis}[
  width=\columnwidth, height=4.8cm,
  xmin=0, xmax=29, ymin=-1.08, ymax=0.85,
  xtick={1,2,4,7,14,21,28},
  xlabel={\footnotesize Graph depth $k$ (layers)},
  ylabel={\footnotesize Spearman with position},
  xticklabel style={font=\scriptsize}, yticklabel style={font=\scriptsize},
  axis line style={inkmuted,line width=0.4pt}, tick style={inkmuted},
  ymajorgrids, grid style={gridgray,line width=0.3pt},
  extra y ticks={0}, extra y tick style={grid style={inkmuted,line width=0.5pt}},
  legend style={font=\scriptsize,draw=none,fill=none,at={(0.5,1.02)},anchor=south,legend columns=3,/tikz/every even column/.append style={column sep=4pt}},
  unbounded coords=jump,
  every axis plot/.append style={line width=1.1pt,mark size=1.9pt},
]
\addplot[color=slotblue,mark=*] table[x=k,y=plr] {data_depth.dat};
\addlegendentry{\rift{}}
\addplot[color=slotviolet,mark=diamond*,mark size=2.3pt] table[x=k,y=widest] {data_depth.dat};
\addlegendentry{widest path}
\addplot[color=slotaqua,mark=square*] table[x=k,y=rollout] {data_depth.dat};
\addlegendentry{rollout}
\addplot[color=slotyellow,mark=diamond,line width=0.7pt,mark size=2.6pt] table[x=k,y=maxflow] {data_depth.dat};
\addlegendentry{attention flow}
\addplot[color=slotmagenta,mark=pentagon*] table[x=k,y=kirchhoff] {data_depth.dat};
\addlegendentry{effective resistance}
\end{axis}
\end{tikzpicture}
\caption{Position dependence as the attention graph deepens (mean of the two 1.7B models). All-path aggregates become functions of position within a few layers; attention flow (hollow diamonds) and effective resistance nearly coincide. \rift{} does not. Negative values mean that earlier tokens score higher.}
\label{fig:depth}
\end{figure}
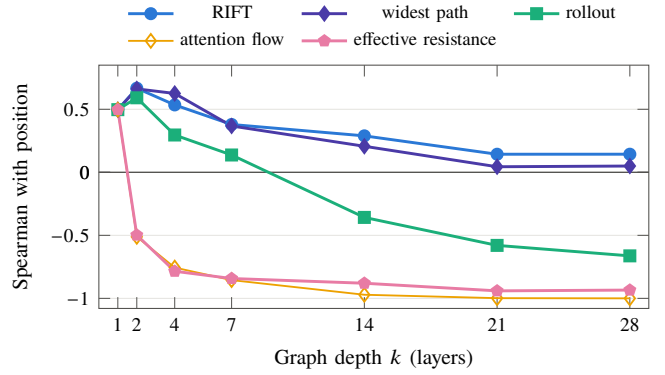

\subsection{The best route points to a candidate read-out layer}
\label{sec:route}
Because Algorithm~\ref{alg:plr} also returns the best route, we can ask where it moves between positions. For the source token that holds the answer (the answer letter, or the answer word in SQuAD), the best route is a single jump in \val{route-single-min}--\val{route-single-max}\% of the \val{route-n} confirmation prompts: it stays on the source's residual stream, moves once to the answer position, and stays there. On the synthetic tasks, the Qwen3 jumps concentrate at one depth, layer \val{route-layer-17} of 28 in both 1.7B models and layer \val{route-layer-4b} of 36 in the 4B model: \val{route-q-mode-pct}\% of the \val{route-q-n} prompts jump there and \val{route-q-near-pct}\% within one layer of it. In Llama the most common layer is \val{route-layer-llama} of 28 (\val{route-llama-mode-pct}\% of prompts; interquartile range up to layer \val{route-layer-llama-hi}). On natural text the depth varies more, with medians at layer \val{route-layer-H} (Qwen3) and \val{route-layer-I} (Llama). This suggests a preferred layer at which the answer position reads the retrieved value, in line with reports of dedicated retrieval heads~\cite{wu2024retrievalhead}. We have not verified this layer causally; \rift{} provides it as a hypothesis at no extra cost.

\subsection{Robustness}
\label{sec:robust}
\emph{Other labels.} For the signed effect of a replacement, \rift{} leads in \val{C1-sgn-wins-of} datasets (\val{C1-sgn-plr-min}--\val{C1-sgn-plr-max} against at most \val{C1-sgn-best-max} for the baselines). The two gradient-based labels, recorded for the synthetic datasets, measure infinitesimal sensitivity rather than the effect of an actual replacement. On these, \rift{} leads in \val{C1-jac-wins-of} and \val{C1-mg-wins-of} datasets, respectively, and rollout is on par or ahead for the instruction-tuned 1.7B and 4B models (datasets B and C).
\emph{Conductances.} On the synthetic datasets, \rift{} works about as well with value-norm conductances (\val{C1-plrval-min}--\val{C1-plrval-max}) and, for the 1.7B models (A and B), with local-Jacobian conductances (\val{A-plrjac} and \val{B-plrjac}); on natural text, value norms are slightly worse (\val{H-plrval} and \val{I-plrval}, versus \val{H-plr} and \val{I-plr}). The plain attention graph therefore suffices, and no gradients are needed.
\emph{Gradients.} On the synthetic datasets, gradient$\times$input, which uses the answer's gradient, reaches only \val{C1-gxi-min}--\val{C1-gxi-max}.

\section{Discussion and Limitations}

\textbf{Practical takeaways.} For decoder LLMs, at least in the settings we test, single-route scores (\rift{} or the widest path) are a better default than attention rollout or attention flow for token attribution: \rift{} costs milliseconds, like rollout, and comes close to the agreement of a replacement experiment that needs a forward pass per token. All-path aggregates mainly report how early a token is. Resistance remains a useful lens, but empirically the informative notion is the series resistance of the best route, not the parallel combination of all routes; why it predicts influence so well (Section~\ref{sec:theory}) remains a hypothesis.

\textbf{Limitations.} Our contexts are short (at most \val{all-len-max} tokens), most prompts (\val{S-total-prompts} of \val{C1-total-prompts}) are synthetic, and every answer is a single token: a letter in the synthetic tasks and a copied word in SQuAD, where we score only questions the model answers correctly; the models come from two families (Qwen3 and Llama 3.2) and have at most 4B parameters; we score influence on the last position only, through single-token replacements, and the scored sources over-represent the relevant records or words. We test the standard, unnormalized rollout and flow, not position-normalized variants~\cite{abnar2020flow,metzger2022general}. Like rollout, \rift{} needs the attention weights of every layer ($L$ matrices of size $T\times T$), which fused attention kernels do not return, and our timings come from a simple NumPy implementation and unbatched forward passes. Longer contexts, multi-token and generated answers, larger models and further model families are the next steps. Natural applications include attributing answers to retrieved passages in retrieval-augmented systems and guiding context compression.

\section{Conclusion}
Viewing a Transformer as a resistor network clarifies how to read token influence from attention. Scores that add up all routes, including attention rollout, attention flow and effective resistance, are dominated by the abundance of parallel routes from early tokens, and so mostly encode position. \rift{}, which follows the path of least resistance, avoids this bias. On held-out synthetic and natural-text data from four models of two families, it predicts the causal effect of replacing a token nearly as well as a second, different replacement does, through a classic shortest-path search over the attention weights of a single forward pass that takes a few milliseconds.

\section*{Acknowledgment}
We disclose that experiment code and analysis scripts were written with the help of an AI coding agent (OpenAI Codex), and that the manuscript text, figures and tables were drafted with the help of an AI assistant (Anthropic Claude). The authors designed the study, directed the experiments, verified all reported numbers against the experiment records, and take full responsibility for the content.

\bibliographystyle{IEEEtranLink}
\bibliography{refs}

\end{document}